\documentclass{article}

\usepackage{microtype}
\usepackage{graphicx}
\usepackage{subcaption}
\usepackage{booktabs} 

\usepackage{hyperref}
\usepackage{xurl}
\usepackage[normalem]{ulem}

\usepackage[preprint]{icml2026}

\usepackage{amsmath}
\usepackage{amssymb}
\usepackage{amsfonts}
\usepackage{mathtools}
\usepackage{amsthm}
\usepackage{dsfont}
\usepackage{algorithm}
\usepackage{algorithmic}

\newcommand{\one}{\mathds{1}}
\newcommand{\argmin}[1]{\underset{#1}{\operatorname{arg}\operatorname{min}} }

\def\argmin{\mathop{\rm argmin}}

\newcommand{\lambd}{\boldsymbol{\lambda}}

\newcommand{\R}{\mathbb{R}}
\newcommand{\E}{\mathbb{E}}
\newcommand{\proba}{\mathbb{P}}

\newcommand{\cS}{\mathcal{S}}
\newcommand{\cF}{\mathcal{F}}
\newcommand{\cY}{\mathcal{Y}} 
\newcommand{\cR}{\mathcal{R}}
\newcommand{\cU}{\mathcal{U}} 
\newcommand{\cD}{\mathcal{D}}

\newcommand{\indep}{\perp \!\!\! \perp}

\usepackage[capitalize,noabbrev]{cleveref}

\theoremstyle{plain}
\newtheorem{theorem}{Theorem}[section]
\newtheorem{proposition}[theorem]{Proposition}

\newtheorem{corollary}[theorem]{Corollary}
\theoremstyle{definition}

\newtheorem{assumption}[theorem]{Assumption}
\theoremstyle{remark}
\newtheorem{remark}[theorem]{Remark}

\usepackage[textsize=tiny]{todonotes}

\icmltitlerunning{Fair regression under localized demographic parity constraints}

\begin{document}

\twocolumn[
  \icmltitle{Fair regression under localized demographic parity constraints}



  \icmlsetsymbol{equal}{*}

  \begin{icmlauthorlist}
    \icmlauthor{Arthur Charpentier}{A1,A2}
    \icmlauthor{Christophe Denis}{equal,A4}
    \icmlauthor{Romuald Elie}{A3}
    \icmlauthor{Mohamed Hebiri}{A3}
    \icmlauthor{Fran\c{c}ois Hu}{A5,A6}
  \end{icmlauthorlist}

    \icmlaffiliation{A1}{Universit\'e du Qu\'ebec \`a Montr\'eal, UQAM, Canada}
    \icmlaffiliation{A2}{Kyoto University, Japan}  
    \icmlaffiliation{A3}{Universit\'e Gustave Eiffel, Paris, France}  
    \icmlaffiliation{A4}{Universit\'e Paris 1 Panth\'eon-Sorbonne, Paris, France}
    \icmlaffiliation{A5}{Milliman Paris, France}
    \icmlaffiliation{A6}{Universit\'e Claude Bernard, Lyon, France}

  \icmlcorrespondingauthor{Christophe Denis}{Christophe.denis1@univ-paris1.fr}

  \icmlkeywords{Fairness, Regression, Demographic parity}

  \vskip 0.3in
]



\printAffiliationsAndNotice{}  

\begin{abstract}

Demographic parity (DP) is a widely used group fairness criterion requiring predictive distributions to be invariant across sensitive groups. While natural in classification, full distributional DP is often overly restrictive in regression and can lead to substantial accuracy loss. We propose a relaxation of DP tailored to regression, enforcing parity only at a finite set of quantile levels and/or score thresholds.
Concretely, we introduce a novel $(\boldsymbol{\ell}, \mathcal{Z})$-fair predictor, which imposes groupwise CDF constraints of the form $F_{f \mid S=s}(z_m) = \ell_m$ for prescribed pairs $(\ell_m, z_m)$. 
For this setting, we derive closed-form characterizations of the optimal fair discretized predictor via a Lagrangian dual formulation and quantify the discretization cost, showing that the risk gap to the continuous optimum vanishes as the grid is refined.
We further develop a model-agnostic post-processing algorithm based on two samples (labeled for learning a base regressor and unlabeled for calibration), and establish finite-sample guarantees on constraint violation and excess penalized risk. In addition, we introduce two alternative frameworks where we match group and marginal CDF values at selected score thresholds. In both settings, we provide closed-form solutions for the optimal fair discretized predictor.
Experiments on synthetic and real datasets illustrate an interpretable fairness–accuracy trade-off, enabling targeted corrections at decision-relevant quantiles or thresholds while preserving predictive performance.

\end{abstract}

\section{Introduction}
\label{sec:intro}

Machine learning systems increasingly support or automate decisions in socially sensitive settings such as credit, hiring, insurance pricing, or public policy. In these applications, predictions may depend (directly or indirectly) on sensitive attributes $S$ (\emph{e.g.,} gender, ethnicity, age), raising major concerns about discrimination. A widely used statistical requirement is {\em demographic parity} (DP), which imposes the predictive distribution to be invariant across groups, \emph{i.e.,} $f(X,S)\indep S$ (or equivalently, the conditional law of $f(X,S)$ given $S$ is the same for all groups). DP is natural and operational in classification, where decisions often boil down to thresholding a score. In regression, however, DP becomes significantly more delicate and typically too restrictive \cite{agarwal2019fair,Chzhen_Denis_Hebiri19,chzhen2020fairwb}.

Indeed, enforcing {\em full} distributional parity in regression often induces a severe loss of accuracy: it constrains the predictor on parts of the outcome distribution that may be irrelevant to the fairness concern, and it may require heavy distortions even when group disparities are localized (\emph{e.g.,} only in upper tails or around key decision thresholds). This is not merely a modeling artifact: in many real deployments, {\em fairness is articulated at a few interpretable summary points} (medians, quartiles, or policy thresholds), rather than on the entire distribution. For instance, pay-transparency regulations often emphasize median and quartile gaps\footnote{E.g., the EU Pay Transparency Directive (EU)~2023/970 and the UK Gender Pay Gap reporting regulations (2017) explicitly require reporting median and quartile statistics.}; in lending, audits commonly focus on approval/denial rates at operational cutoffs\footnote{E.g., in the US, ECOA/Regulation~B compliance is typically monitored through comparative acceptance/denial rates across protected classes.}. These examples suggest that {\em quantile-level} or {\em threshold-level} parity may be a more faithful and actionable target than full DP.

Motivated by this, we study {\em localized relaxations of demographic parity} for regression. Rather than enforcing equality of the entire conditional distribution of $f(X,S)$ across groups, we impose fairness only at a finite set of probability levels (quantiles) and/or score thresholds. This viewpoint connects to recent work arguing that ``quantile fairness'' captures important distributional disparities that are invisible to mean-based criteria and can be enforced through post-processing or calibration \cite{yang2019fairqr,liu2022cfqr,wang2023eoc,plecko2020fairadapt}. Concretely, given a vector of quantiles $\boldsymbol{\ell}=(\ell_1,\dots,\ell_M)$ and/or a vector of thresholds $\mathcal{Z}=(z_1,\dots,z_M)$, we
consider constraints of the form
\[
\begin{cases}
    Q_{f\,|\,S=s}(\ell_m)\ \text{is equal across } s = Q_{f}(\ell_m),\\
F_{f\,|\,S=s}(z_m)\ \text{is equal across } s=F_{f}(z_m),
\end{cases}
\]
These constraints are low-dimensional, interpretable, and naturally aligned with how stakeholders specify fairness requirements (\emph{e.g.,} ``equal predicted median'', ``equal top-decile access'', or ``equal approval rate at cutoff $z$'').

\paragraph{Related work.}

Recent work argues that enforcing fairness uniformly over the whole score range can be unnecessarily stringent.
\cite{he2025enforcing} proposes to enforce a so-called ``\emph{partial fairness}'' only on score ranges of interest (e.g., contested regions), using an in-processing formulation with difference-of-convex constraints.
In a different direction, \cite{chen2025testing} develops a hypothesis-testing framework that audits approximate (strong) demographic parity under explicit utility trade-offs using  Wasserstein projections, motivated by causal policy evaluation.
\cite{wang2023eoc} introduces Equal Opportunity of Coverage and uses binned fair quantile regression as a post-processing step.

\paragraph{Our framework.}

We introduce a new notion of fairness that formalizes this idea.
We define the {\em $(\boldsymbol{\ell},\mathcal{Z})$-fair predictor}, which directly imposes a finite family of quantile/threshold constraints coupling the probability levels $\boldsymbol{\ell}$ and the corresponding score values $\mathcal{Z}$. This notion yields a continuous {\em fairness--accuracy continuum}: increasing $M$ strengthens fairness (and approaches full distributional parity), while small $M$ targets only the distributional regions of interest.
Additionally, we introduce specific settings that focus on matching group and marginal CDF values at selected thresholds. These can be viewed as variants of the $(\boldsymbol{\ell}, \mathcal{Z})$-fair constraint, sharing the same objective: localizing the fairness constraint to mitigate accuracy loss.


\paragraph{Why quantiles rather than optimal transport?}
A prominent alternative for distributional fairness in regression aligns group-conditional predictive distributions via Wasserstein barycenters or optimal transport (OT) mappings.
We refer 
to \cite{gordaliza2019obtaining,chzhen2020fairwb} for Wasserstein/OT-based approaches to distributional fairness.
These approaches are elegant and can provide strong guarantees, but they may be brittle in practice: they depend on a choice of ground cost, can be sensitive to outliers/heavy tails, and may behave poorly under multimodality or group imbalance. In contrast, quantile-based constraints reduce fairness to a finite set of {\em univariate} restrictions.
They are robust, invariant under monotone transformations of the score, computationally simple, and directly interpretable in terms of policy-relevant thresholds. This complements recent OT-based fairness lines developed by some of the authors (\emph{e.g.,} multi-task/barycentric formulations or sequentially fair mechanisms) \cite{denis2024multiclassdp,hu2024sequential,hu2023mtlwb,charpentier2022quantifying,hu2023parametric}.

\paragraph{Contributions.}
Our contributions are threefold:
(i) we introduce  localized-based relaxations of DP for regression. The {\em $(\boldsymbol{\ell},\mathcal{Z})$-fair predictor} and variants of this setting (the {\em partially DP-fair discretized predictors});
(ii) we characterize the corresponding optimal fair predictors and clarify their relation to full DP; and (iii) we propose a practical, data-driven procedure and demonstrate theoretically and empirically (synthetic and real data) that localized-based fairness can mitigate distributional bias while preserving predictive performance.


\section{Statistical setting}
\label{sec:statSetting}

\paragraph{Data and risk.}
Let $(X,S,Y)$ be a random triplet with features $X\in\R^d$, sensitive attribute $S\in\cS$ (multi-group setting, $|\cS|<\infty$), and response $Y\in\R$ with $\E[Y^2]<\infty$. 
Denote $\pi_s:=\proba(S=s)$, and use the shorthand $\proba_s(\cdot):=\proba(\cdot\mid S=s)$ and $\E_s[\cdot]:=\E[\cdot\mid S=s)$.
We work under squared loss and define the (population) risk
\[
R(f) := \E\left[(Y-f(X,S))^2\right]
       = \sum_{s\in\cS}\pi_s\,\E_s\left[(Y-f(X,s))^2\right].
\]
The Bayes predictor for squared loss is the conditional mean $f^\star(x,s):=\E[Y\mid X=x,S=s]$ (hence $f^\star\in\argmin_f R(f)$), see, e.g., \cite{hastie2009elements}. 
Our post-processing is group-conditional and therefore assumes $S$ is available at deployment (as is common in fairness auditing/calibration pipelines).
Throughout, we assume bounded outcomes: there exists $A>0$ such that $|Y|\leq A$ a.s., which implies $|f^\star(X,S)|\le A$ a.s. (by Jensen).
Equivalently, we may write the regression model
\[
Y  = f^\star(X,S)+\varepsilon,
\qquad \E[\varepsilon\mid X,S]=0.
\]

\paragraph{Non-atomicity.}
To avoid ties at the fairness thresholds and ensure well-defined quantile-level constraints, we assume that the group-conditional distribution of $f^\star(X,S)$ is continuous.

\begin{assumption}[Continuity / non-atomicity]
\label{ass:cont}
For every $s\in\cS$, the CDF \ $t\mapsto \proba_s\left(f^\star(X,s)\le t\right)$ is continuous.
\end{assumption}

\paragraph{Predictor classes and discretization.}
Let $\cF$ be the set of all measurable predictors $f:\R^d\times\cS\to[-A,A]$.
For $K\ge2$, introduce a regular grid of $[-A,A]$ with $K$ points
\[
\cY_K:=\left\{y_k:  y_k=-A+\frac{2A}{K-1}(k-1),\ k\in[K]\right\}.
\]
We then define the discretized class
$\cF_K:=\{f:\R^d\times\cS\to\cY_K\}$.
This discretization is standard 
when dealing with real-valued predictors: 
it yields finite-dimensional constraints and closed-form characterizations, while the induced approximation error vanishes as the grid is refined (see Proposition~\ref{prop:discrCost} and classical quantization results, \emph{e.g.,} \cite{gray1998quantization, Agarwal_Beygelzimer_Dubik_Langford_Wallach18}).

\subsection{$(\boldsymbol{\ell},\mathcal{Z})$-fair predictor}
\label{subsec:LZFair}

Throughout, we fix $M\ge 1$ and assume
$\boldsymbol{\ell}=(\ell_1,\dots,\ell_M)\in(0,1)^M$ with
$0<\ell_1<\cdots<\ell_M<1$.
We also fix thresholds $\mathcal Z=(z_1,\dots,z_M)\in[-A,A]^M$
with $z_1<\cdots<z_M$ (and typically $z_m\in\cY_K$ in the discretized setting).

\paragraph{Quantile/threshold constraints.}
A discretized predictor $f\in\cF_K$ is said to be {\em $(\boldsymbol{\ell},\mathcal{Z})$-fair} if
\begin{equation}
\label{eq:LZconstraint}
\forall s\in\cS,\ \forall m\in[M],~
\proba_s\big(f(X,s)\le z_m\big) = \ell_m.
\end{equation}
Equivalently, $F_{f\mid S=s}(z_m)=\ell_m$ for all $m$, \emph{i.e.,} each group has the same CDF values at the specified thresholds. This can be viewed as a finite relaxation of distributional demographic parity, which is known to be demanding in regression \cite{agarwal2019fair,chzhen2020fairwb}.


Let $s\in\cS$ and set $T:=f(X,s)$, with conditional CDF\ $F_s(t):=\proba(T\le t\mid S=s)$ and quantile function $Q_s(\tau):=\inf\{t\in\R:\,F_s(t)\ge \tau\}$.
Assume that $T\mid(S=s)$ is absolutely continuous with density $p_s$ such that $p_s(t)>0$ for Lebesgue-a.e. $t$ in an open interval containing $z_m$, so that $F_s$ is continuous and strictly increasing in a neighborhood of $z_m$.
Then, the quantile function coincides with the usual inverse, $Q_s(\tau)=F_s^{-1}(\tau)$ for all $\tau\in(0,1)$, and therefore, for any $m\in[M]$ and any $\ell_m\in(0,1)$, 
\begin{equation}
\label{eq:FQ_equiv}
F_s(z_m)=\ell_m
~~\Longleftrightarrow~~
Q_s(\ell_m)=z_m.
\end{equation}
Consequently, the constraints $F_{f\mid S=s}(z_m)=\ell_m$ can be interpreted either as prescribing groupwise CDF values at thresholds $\mathcal Z$, or equivalently as fixing group-conditional quantiles at levels $\boldsymbol{\ell}$. In other words, our framework can be interpreted either as enforcing parity at prescribed thresholds $\mathcal Z$ or as aligning specified quantiles at levels $\boldsymbol{\ell}$.

We study the risk-minimizing discretized predictor under these constraints:
\begin{equation}
\label{eq:optFair}
f^\star_{(\boldsymbol{\ell},\mathcal{Z})\text{-fair}}
\in \argmin_{f\in\cF_K}\Big\{R(f): f \text{ satisfies \eqref{eq:LZconstraint}}\Big\}.
\end{equation}

\paragraph{Lagrangian form.}
For $y\in[-A,A]$, define the indicator vector
\[
a(y):=\big(\one\{y\le z_1\},\dots,\one\{y\le z_M\}\big)\in\{0,1\}^M.
\]
Let $\lambda=(\lambda_{s,m})_{s\in\cS,m\in[M]}$ and write $\lambda_s\in\R^M$ for the block
$(\lambda_{s,1},\dots,\lambda_{s,M})$. The (groupwise) Lagrangian is
\[
\cR_\lambda(f)
:=R(f)+\sum_{s\in\cS}\sum_{m\in[M]}\lambda_{s,m}
\Big(\mathbb{P}_s(f(X,s)\le z_m)-\ell_m\Big),
\]
a standard constrained-risk formulation (see, e.g., \cite{boyd2004convex}).

\begin{theorem}[Optimal $({\boldsymbol{\ell}},\mathcal{Z})$-fair discretized predictor]
\label{thm:optimLZFair}
Under Assumption~\ref{ass:cont}, there exists $\lambda^\star$ such that the predictor $f^\star_{(\boldsymbol{\ell},\mathcal{Z})\text{-fair}}$ admits the pointwise form
\begin{equation}
\label{eq:eqlZOpt}
f^\star_{(\boldsymbol{\ell},\mathcal{Z})\text{-fair}}(x,s)
\in \argmin_{y\in\cY_K}\Big\{
\pi_s\,(y-f^\star(x,s))^2 + \langle \lambda^\star_s,\ a(y)\rangle
\Big\}.
\end{equation}
Moreover, $\lambda^\star$ can be chosen as a minimizer of the dual objective
\begin{equation}
\label{eq:dual}
\lambda^\star \in \argmin_{\lambda\in\R^{|\cS|\times M}}
\sum_{s\in\cS}\E_s\left[ V_s(\lambda;X)\right],
\end{equation}
where $V_s(\lambda;x):=\displaystyle\max_{y\in\cY_K}\lbrace\Phi_{s,\lambda}(x,y)\rbrace$ and
\[
\Phi_{s,\lambda}(x,y):= -\pi_s\,(y-f^\star(x,s))^2 - \langle \lambda_s,\ a(y)-\boldsymbol{\ell}\rangle .
\]
\end{theorem}

\paragraph{Penalized-risk interpretation.}
Theorem~\ref{thm:optimLZFair} implies that the optimal fair predictor is also a minimizer of a Lagrangian-penalized risk.

\begin{corollary}
Under Assumption~\ref{ass:cont},
\[
f^\star_{(\boldsymbol{\ell},\mathcal{Z})\text{-fair}}
\in \argmin_{f\in\cF_K}\ \cR_{\lambda^\star}(f).
\]
\end{corollary}

\paragraph{Discretization cost.}
Define the (continuous) constrained optimum
\[
\tilde f \in \argmin_{f\in\cF}\Big\{R(f): f \text{ satisfies \eqref{eq:LZconstraint}}\Big\}.
\]
We compare the optimal discretized risk to its continuous counterpart.

\begin{proposition}[Cost of discretization]
\label{prop:discrCost}
The following holds:
\[
R\left(f^\star_{(\boldsymbol{\ell},\mathcal{Z})\text{-fair}}\right)-R(\tilde f)
 \le  \frac{CA^2}{K},
\]
for some absolute constant $C>0$.
Consequently,
\[
R\left(f^\star_{(\boldsymbol{\ell},\mathcal{Z})\text{-fair}}\right)\to R(\tilde f)
\qquad\text{as }K\to+\infty.
\]
\end{proposition}

\section{Data-driven algorithm}
\label{sec:algorithm}

This section describes a practical procedure to estimate the optimal $(\boldsymbol{\ell},\mathcal{Z})$-fair discretized predictor characterized in Theorem~\ref{thm:optimLZFair}. Our approach is a \emph{post-processing} method: we first learn an unconstrained regressor and then calibrate its outputs to satisfy the fairness constraints. Post-processing is model-agnostic and can be applied to any black-box regressor \cite{hardt2016equality,chzhen2020fairwb}.

\paragraph{Two-sample setup.}
We use two independent samples:\\
$\bullet$ a labeled sample $\cD_n=\{(X_i,S_i,Y_i)\}_{i=1}^n$, used to learn a base regressor $\hat f$
for $f^\star$;\\
$\bullet$ an unlabeled sample $\cD_N=\{(X'_i,S'_i)\}_{i=1}^N$, used to estimate the dual parameters (Lagrange multipliers) enforcing the fairness constraints.

Using unlabeled data for calibration is natural here because the constraints $\proba(f(X,S)\le z_m\mid S=s)=\ell_m$ depend only on the distribution of $(X,S)$ and on the predictor outputs, not directly on $Y$.

\paragraph{Dithering to ensure continuity.}
Assumption~\ref{ass:cont} avoids ties at thresholds and ensures a well-behaved dual. In practice, $\hat f$ may have atoms (\emph{e.g.,} tree-based models). We therefore introduce a randomized (``dithered'')
version
\[
\bar f(x,s):=\Pi_{[-A,A]}\big(\hat f(x,s)+\xi\big),
\text{ for } \xi\sim \mathrm{Unif}([0,u]),
\]
where $\Pi_{[-A,A]}$ denotes projection onto $[-A,A]$ and $\xi$ is independent of all data.
Conditionally on $\cD_n$, the mapping $t\mapsto \proba(\bar f(X,S)\le t\mid S=s)$ is continuous for each $s$, which simplifies both theory and implementation.

The dithering variable $\xi$ is introduced only to break ties and guarantee continuity. When $\hat f$ is continuous (or when ties are negligible), we set $u=0$ and the procedure becomes deterministic.
Otherwise, $u$ can be chosen arbitrarily small, so that the impact of randomization on predictions is negligible.

\paragraph{Empirical group weights.}
From $\cD_N$, define $\hat\pi_s:=N_s/N$ with
\[
N_s:=\sum_{i=1}^N \one\{S'_i=s\}, ~~
I_s:=\{i\in[N]:S'_i=s\},
\]
and let $\pi_{\min}:=\displaystyle\min_{s\in\cS}\{\pi_s\}>0$.
Recall $\mathcal Z=(z_1,\dots,z_M)$ and $\boldsymbol{\ell}=(\ell_1,\dots,\ell_M)$.
Define for $y\in[-A,A]$
\[
\begin{cases}
    a(y):=\big(\one\{y\le z_1\},\dots,\one\{y\le z_M\}\big)\in\{0,1\}^M,
\\
b(y):=a(y)-\boldsymbol{\ell}.
\end{cases}
\]
For $\lambda=(\lambda_{s,m})_{s\in\cS,m\in[M]}$, write $\lambda_s:=(\lambda_{s,1},\dots,\lambda_{s,M})$.
Define the empirical per-sample dual score, for $y\in\cY_K$,
\[
\widehat\Phi_{s,\lambda}(x,y)
:= -\hat\pi_s\,(y-\bar f(x,s))^2 - \langle \lambda_s,\ b(y)\rangle.
\]
The empirical dual objective (compared with the population dual in Theorem~\ref{thm:optimLZFair}) is
\begin{equation}
\label{eq:emp_dual}
\widehat H(\lambda)
=\sum_{s\in\cS}\frac{1}{N_s}\sum_{i\in I_s}\max_{y\in\cY_K}\widehat\Phi_{s,\lambda}(X'_i,y).
\end{equation}
Since $\widehat H$ is a sum of pointwise maxima of affine functions in $\lambda$, it is convex and can be minimized with standard first-order methods (\emph{e.g.,} projected subgradient) \cite{boyd2004convex,shalev2014understanding}.
We define the estimated multipliers as any minimizer
\[
\hat\lambda \in \argmin_{\lambda\in\R^{|\cS|\times M}}\ \widehat H(\lambda).
\]

\paragraph{Calibrated fair predictor.}
Finally, the empirical $(\boldsymbol{\ell},\mathcal{Z})$-fair post-processed predictor is
\begin{equation*}
\hat f_{(\boldsymbol{\ell},\mathcal{Z})\text{-fair}}(x,s)
\in \argmin_{y\in\cY_K}
\Big\{\hat\pi_s\,(y-\bar f(x,s))^2 + \langle \hat\lambda_s,\ a(y)\rangle\Big\}.
\end{equation*}
This mirrors the population characterization in Theorem~\ref{thm:optimLZFair}, with $f^\star$ replaced by $\bar f$ and $\lambda^\star$ replaced by $\hat\lambda$.

\subsection{Theoretical study}

We summarize the main statistical guarantees satisfied by the post-processed predictor $\hat f_{(\boldsymbol{\ell},\mathcal{Z})\text{-fair}}$.

\paragraph{Constraint violation.}
For any predictor $f$, define the maximal constraint violation
\[
\cU_{(\boldsymbol{\ell},\mathcal{Z})}(f)
:=\max_{s\in\cS}\left\{\max_{m\in[M]}
\Big|\proba_s\big(f(X,s)\le z_m\big)-\ell_m\Big|\right\}.
\]

\begin{theorem}[Rate for fairness violation]
\label{thm:unfairnessRate}
There exists a constant $C_\cS$ (depending only on $\cS$ and $\pi_{\min}$) such that
\[
\E\left[\cU_{(\boldsymbol{\ell},\mathcal{Z})}
\left(\hat f_{(\boldsymbol{\ell},\mathcal{Z})\text{-fair}}\right)\right]
\le C_\cS \ \left(\sqrt{\frac{1}{N}} + \dfrac{K^2}{N} \right).
\]
\end{theorem}

Several comments can be made from this results. 
First, the bound depends only on the unlabeled sample and is independent of the quality of the initial estimator $\hat f$, implying that the result holds for any base regression algorithm. Second, it guarantees that the empirical fair predictor asymptotically satisfies the target fairness constraints provided that $K^2/N \rightarrow 0$.
Third, the bound decomposes into two terms: the first arises from controlling the deviation between the true CDF and the empirical CDF, while the second accounts for tie effects due to minimizing the empirical counterpart of the function $H$. Finally, the obtained rates highlight a trade-off between the grid resolution and the size of the unlabeled sample. 
From this result, we also derive a high-probability guarantee.
\begin{theorem}[High-probability fairness violation]
\label{thm:unfairness_hp}
Assume $|Y|\le A$ a.s. and $\pi_{\min}=\min_s\pi_s>0$.
Conditionally on $\cD_n$, for any $\delta\in(0,1)$, with probability at least $1-\delta$ (over $\cD_N$ and the dithering),
\[
\cU_{(\boldsymbol{\ell},\mathcal Z)}\!\left(\hat f_{(\boldsymbol{\ell},\mathcal Z)\text{-fair}}\right)
\le C_\cS\, \left(\sqrt{\frac{1}{N}}+ \dfrac{K^2}{N} + \sqrt{\frac{\log(1/\delta)}{N}}\right).
\]
\end{theorem}

\paragraph{Excess penalized risk.}
Recall the Lagrangian-penalized risk $\cR_{\lambda^\star}$ introduced in Section~\ref{sec:statSetting}. The next bound controls the excess penalized risk of the empirical post-processing solution relative to the population optimum.

\begin{theorem}[Excess penalized risk]
\label{thm:risk}
There exists a constant $C_\cS$ such that
\begin{multline*}
\E\Big[
\cR_{\lambda^\star}\big(\hat f_{(\boldsymbol{\ell},\mathcal Z)\text{-fair}}\big)
-\cR_{\lambda^\star}\left(f^\star_{(\boldsymbol{\ell},\mathcal Z)\text{-fair}}\right)
\Big] \le \\  C_{\cS,A} \Big(
\E\big[|\hat f(X,S)-f^\star(X,S)|\big]
 + \sqrt{\dfrac{1}{{N}}}+ u \Big) \\+ C_{\cS}M\ \left(\sqrt{\frac{1}{N}} + \dfrac{K^2}{N} \right).
\end{multline*}
\end{theorem}
The theorem shows that the excess risk decomposes into two main components. The first one consists in 
three error terms: 
(i) the statistical error of the base regressor $\hat f$,
(ii) the statistical error of the estimators $(\hat{\pi}_s)_{s \in \cS}$, 
(iii) the dithering level $u$ introduced to ensure continuity.
The second component is related to the unfairness of the predictor and corresponds to the calibration error due to estimating the dual with $N$ unlabeled points and a grid of size $K$.
\begin{theorem}[High-probability excess penalized risk]
\label{thm:risk_hp}
Under the assumptions of Theorem~\ref{thm:unfairness_hp}, for any $\delta\in(0,1)$, with probability at least $1-\delta$ (conditionally on $\cD_n$),
\begin{multline*}
\cR_{\lambda^\star}\!\left(\hat f_{(\boldsymbol{\ell},\mathcal Z)\text{-fair}}\right)
-\cR_{\lambda^\star}\!\left(f^\star_{(\boldsymbol{\ell},\mathcal Z)\text{-fair}}\right)\\
 \le  C_{\cS,A} \Big(\E[|\hat f(X,S)-f^\star(X,S)|] +\sqrt{\dfrac{1}{N}}+u\Big)
\\ + \mathcal{C}_{\cS}M\left(\sqrt{\frac{1}{N}} + \dfrac{K^2}{N}\sqrt{\frac{\log(1/\delta)}{N}}\right).
\end{multline*}
\end{theorem}

\begin{remark}[High-probability variants]
\label{rem:hp}
Conditionally on the labeled sample $\cD_n$, the calibration step depends only on $\cD_N$.
Since the constraints only involve CDF values at thresholds, one may  control $\max_{m\in[M]}|\hat F_s(z_m)-F_s(z_m)|$ using the Dvoretzky--Kiefer--Wolfowitz inequality
(in its sharp form due to Massart) \cite{dvoretzky1956dkw,massart1990dkw}.
This yields high-probability bounds with $\sqrt{\log(1/\delta)/N_s}$ dependence. A detailed statement is given in Appendix~\ref{app:B:section:4}.


\end{remark}

\subsection{Implementation details and practical choices}\label{subsec:impl}

\paragraph{Choosing $(\boldsymbol{\ell},\mathcal Z)$.}
We assume $0<\ell_1<\cdots<\ell_M<1$ and $z_1<\cdots<z_M$.
In practice, $(\boldsymbol{\ell},\mathcal Z)$ can be selected to match either
(i) policy targets (\emph{e.g.,} fixed acceptance/flagging cutoffs), or
(ii) distributional summaries (\emph{e.g.,} medians/upper quantiles).

\paragraph{Choosing $K$ and $u$.}
The grid size $K$ trades off computational cost, discretization error, and unfairness rate:
Proposition~\ref{prop:discrCost} yields a risk gap of order $A^2/K$, while Theorem~\ref{thm:unfairnessRate} gives a bound of order
$\sqrt{1/N} + K^2/N$. Therefore a choice of $K= N^{1/3}$ trades-off
the discretization error and unfairness rate.
The dithering level $u$ is only used to avoid ties and ensure continuity; when the base regressor
$\hat f$ is (approximately) continuous, we set $u=0$, otherwise we take $u$ small (\emph{e.g.,} $u\ll 1$).

\paragraph{Optimizing the dual.}
The objective $\widehat H$ in~\eqref{eq:emp_dual} is convex (as a sum of maxima of affine functions),
so it can be minimized with standard first-order methods (projected subgradient or mirror descent).
Each evaluation of $\widehat H(\lambda)$ requires $\mathcal{O}(NK)$ operations, since the inner maximization
is over the $K$ grid points.

\section{Particular setting: Partially DP-fair discretized predictor}
\label{sec:partialDP}

This section studies a special case of our framework, where fairness is imposed by matching group-conditional and marginal probabilities at a finite set of thresholds.



\subsection{$\mathcal{Z}$-DP (partial demographic parity at thresholds)}

Fix an integer $M\ge1$ and a strictly increasing vector of thresholds $\mathcal{Z}=(z_1,\dots,z_M)\in[-A,A]^M$ (typically with $z_m\in\cY_K$).
For a discretized predictor $f\in\cF_K$, we say that $f$ is {\em $\mathcal{Z}$-DP-fair} if $\forall s\in\cS,\ \forall m\in[M],$
\begin{equation}
\label{eq:ZDP}
\proba_s\big(f(X,s)\le z_m\big) = \proba\big(f(X,S)\le z_m\big).
\end{equation}
That is, at each threshold $z_m$, every group shares the same fraction of predictions below $z_m$ as in the overall population. 

We consider the risk-minimizing predictor under \eqref{eq:ZDP}:
\begin{equation}
\label{eq:optZfair}
f^\star_{\mathcal{Z}\text{-fair}}
\in \argmin_{f\in\cF_K}\Big\{R(f): f \text{ satisfies \eqref{eq:ZDP}}\Big\}.
\end{equation}

\begin{corollary}[Recovery of discretized strong DP]
If $\mathcal Z=\cY_K$ (equivalently, constraints at all grid points),
then $\mathcal Z$-DP fairness is equivalent to equality of the entire
discretized predictive distributions across groups.
\end{corollary}

\paragraph{Dual constraints and notation.}
The constraints \eqref{eq:ZDP} compare each group to the marginal distribution; equivalently, they can be written as $\sum_{s\in\cS}\pi_s\,\proba_s(f\le z_m)-\proba_s(f\le z_m)=0$.
This yields a dual where the Lagrange multipliers at each threshold must sum to zero across groups. We therefore define
\[
\Delta_M
:=\Big\{\lambda\in\R^{|\cS|\times M}:\ \sum_{s\in\cS}\lambda_{s,m}=0,\ \forall m\in[M]\Big\}.
\]
As before, let $a(y)=(\one\{y\le z_1\},\dots,\one\{y\le z_M\})\in\{0,1\}^M$ and
$\lambda_s=(\lambda_{s,1},\dots,\lambda_{s,M})$.

\begin{theorem}[Optimal $\mathcal{Z}$-DP-fair discretized predictor]
\label{thm:OptimFair}
Under Assumption~\ref{ass:cont}, there exists $\lambda^\star\in\Delta_M$ such that
\begin{equation}
\label{eq:Zfair_pointwise}
f^\star_{\mathcal{Z}\text{-fair}}(x,s)
\in \argmin_{y\in\cY_K}
\Big\{\pi_s\,(y-f^\star(x,s))^2+\langle \lambda^\star_s,\ a(y)\rangle\Big\}.
\end{equation}
Moreover, $\lambda^\star$ can be chosen as a minimizer of the dual objective
\begin{equation}
\label{eq:Zfair_dual}
\lambda^\star \in \argmin_{\lambda\in\Delta_M}
\sum_{s\in\cS}\E_s\!\left[\max_{y\in\cY_K}\Phi^{\mathcal Z}_{s,\lambda}(X,y)\right],
\end{equation}
where $\Phi^{\mathcal Z}_{s,\lambda}(x,y)
:= -\pi_s\,(y-f^\star(x,s))^2-\langle \lambda_s,\ a(y)\rangle $.
\end{theorem}

\subsection{$\partial{\mathcal{Z}}$-DP (partial demographic parity with borders constraint)}
We start by two couples of interest in the quantile/threshold space. Formally, we introduce $\boldsymbol{\ell}=(\ell_1,\ell_2)\in(0,1)^2$ and $\mathcal{Z}=(z_1,z_2) \in [-A,A]^2$. For a discretized predictor $f\in\cF_K$, the ultimate goal is to achieves $\partial{\mathcal{Z}}$-DP: 
\begin{eqnarray*}
\forall s\in\cS,
\ \forall t\in(z_1,z_2),~ 
\proba_s\big(f(X,s)\le t \big) = \proba\big(f(X,S)\le t \big),
\end{eqnarray*}
with additionally $\proba_s\big(f(X,s) = z_m\big) = \ell_m$ for $m = 1,2$ for all $s\in\cS$. This definition enforces equality of the CDF across groups in the interval $[z_1,z_2]$ while imposing the level of quantiles $\ell_1$ and $\ell_2$ at the borders of this interval. That is, we explicitly control the mass of the CDFs in the interval $[z_1,z_2]$ of interest.
The solution of the problem 
\begin{equation}
\label{eq:optZfairBorder}
f^\star_{\partial\mathcal{Z}\text{-fair}}
\in \argmin_{f\in\cF_K}\Big\{R(f): f \text{ satisfies } \partial{\mathcal{Z}}\text{-DP}\Big\},
\end{equation}
can be approached by combining ideas from the $(\boldsymbol{\ell},\mathcal{Z})$-DP framework and the above $\mathcal{Z}$-DP one. To this end, we discretize $[z_1,z_2]$ and define $\widetilde{\mathcal{Z}}^M_{z_1,z_2} = (\tilde{z}_1,\ldots,\tilde{z}_M)\in [-A,A]^M$ with $z_1 = \tilde{z}_0 < \tilde{z}_1 <\ldots < \tilde{z}_M < \tilde{z}_{M+1} = z_2$.

We consider a proxy of the above $\partial{\mathcal{Z}}$-DP fairness constraint and ask for $\proba_s\big(f(X,s)\le z \big) = \proba\big(f(X,S)\le z \big)$ only for thresholds $z\in \widetilde{\mathcal{Z}}^M_{z_1,z_2} $. Hence our goal becomes
\begin{equation}
\label{eq:optZfairBorder}
\tilde{f}^\star_{\partial\mathcal{Z}^M\text{-fair}}
\in \argmin_{f\in\cF_K}\Big\{R(f): f \text{ satisfies } \partial{\widetilde{\mathcal{Z}}^M_{z_1,z_2}}\text{-DP}\Big\},
\end{equation}
where a discretized prediction function $f\in \mathcal{F}_K$ is said $\partial{\widetilde{\mathcal{Z}}^M_{z_1,z_2}}$-DP fair if 
\begin{equation*}
    \forall s\in\cS,
\ \forall z\in\widetilde{\mathcal{Z}}^M_{z_1,z_2} ,~ 
\proba_s\big(f(X,s)\le z \big) = \proba\big(f(X,S)\le z \big)
\end{equation*}
and $\proba_s\big(f(X,s) = z_m\big) = \ell_m$ for $m = 1,2$.

\begin{theorem}[Proxy $\partial\mathcal{Z}$-DP discretized predictor]
\label{thm:proxyBorderPartialDP}
Under Assumption~\ref{ass:cont}, there exists $\lambda_1^\star \in \R^{|\cS|\times 2}$ and $\lambda_2^\star \in \Delta_M$ such that the predictor $\tilde{f}^\star_{ \partial \mathcal{Z}^M\text{-fair}}$ admits the pointwise form given by
\begin{multline*}
\tilde{f}^\star_{\partial\mathcal{Z}^M\text{-fair}}(x,s)
\in \argmin_{y\in\cY_K}
\Big\{\pi_s\,(y-{f}^\star(x,s))^2 \Big. \\
\Big. +  \langle (\lambda_1^\star)_s,\ a(y)\rangle +\langle (\lambda_2^\star)_s,\ \tilde{a}(y)\rangle\Big\}.
\end{multline*}
with $(\lambda_1^\star,\lambda_2^\star)$ being a minimizer of the dual objective
\begin{equation*}
(\lambda_1^\star,\lambda_2^\star) \in \argmin_{(\lambda_1 ,\lambda_2) \in\R^{|\cS|\times 2} \times \Delta_M}
\sum_{s\in\cS}   \left[\max_{y\in\cY_K}\Phi^{\widetilde{\mathcal{Z}}_{z_1,z_2}}_{s,\lambda}(X,y)\right].
\end{equation*}
with 
\begin{multline*}
\Phi^{\widetilde{\mathcal{Z}}_{z_1,z_2}}_{s,\lambda}(x,y):= -\pi_s\,(y-f^\star(x,s))^2 \\ 
- \langle (\lambda_1)_s,\ a_1(y)-\boldsymbol{\ell}\rangle -\langle (\lambda_2)_s,\ {a}_2(y)\rangle ,
\end{multline*}
and $a_1(y) = (\one\{y\le z_1\},\one\{y\le {z}_2\})\in\{0,1\}^2$ and ${a}_2(y)=(\one\{y\le \tilde{z}_1\},\dots,\one\{y\le \tilde{z}_M\})\in\{0,1\}^M$.
\end{theorem}
Theorem~\ref{thm:proxyBorderPartialDP} exhibits a solution $\tilde{f}^\star_{\partial\mathcal{Z}^M\text{-fair}}$ that is a good proxy for $f^\star_{\partial\mathcal{Z}\text{-fair}}$ from Equation~\eqref{eq:optZfairBorder} when the grid $\widetilde{\mathcal{Z}}^M_{z_1,z_2}$ is good, \emph{e.g.,} a regular grid with large $M$. From the estimation perspective, building a data-driven method from $\tilde{f}^\star_{\partial\mathcal{Z}^M\text{-fair}}$ is performed as in the previous section -- a labeled dataset to estimateur the regression function $f^*$ and an unlabeled dataset to calibrate the partial unfairness.

The framework that we consider here resembles the one in~\cite{he2025enforcing} where the authors fits the CDFs across groups for a range of quantiles --- $[\ell_1,\ell_2]$ with our notation. The only difference is that we also specify the range of prediction values $[z_1,z_2]$. In terms of estimation strategy we also differ since we rely on post-processing while they consider in-processing approaches --- exploiting discretization as well.

\section{Numerical experiments}
\label{sec:experiments}

In this section, we validate our framework on both real and synthetic data designed to highlight the trade-off between predictive risk and distributional fairness constraints. We illustrate how our approach supports a continuum of interventions, from \emph{surgical} corrections at a few policy-relevant thresholds to \emph{localized} regional constraints, and we contrast these with a fully distribution-matching (``strong DP'') baseline.

We consider a regression setting where the sensitive group $S \in \{A, B\}$ influences the target $Y$ through both a location shift and a group-specific non-linearity.

\paragraph{Synthetic data and base learner.}
We generate $n = 4000$ samples $(X, S, Y)$ for each of $N_{sim} = 30$ simulations, with $X \sim \mathcal{U}(0, 10)^2$ and $\mathbb{P}(S=B) = 0.5$. The outcome follows $Y = f^*(X, S) + \varepsilon$, where $\varepsilon \sim \mathcal{N}(0, 5)$ and:
\begin{equation*}
    f^*(X, S) = 5X_1 + 3X_2 + 20 + \mathds{1}_{\{S=B\}} \left( 15 + 2(X_1 - 5)^2 \right).
\end{equation*}
This model induces a linear location shift ($+15$) and a non-linear structural polarization ($2(X_1-5)^2$) for group $B$. The unconstrained predictor $\hat{f}$ is estimated using a decision tree regressor (minimum 20 samples per leaf). All outcomes/predictions are clipped to $[-100, 100]$ and post-processing is evaluated on a regular grid of size $K=201$.

\paragraph{Compared methods.}
All methods below are applied as post-processing on top of the same base regressor $\hat f$:
(i) \textbf{Unconstrained}: the base regressor $\hat f$ (no fairness post-processing);
(ii) \textbf{$(\boldsymbol{\ell},\mathcal Z)$-fair}: enforce $F_{\hat f\mid S=s}(z_m)=\ell_m$ at a small number of prescribed pairs $(\ell_m,z_m)$ (Figure~\ref{fig:synthetic_data_lZfair});
(iii) \textbf{$\mathcal Z$-fair}: enforce partial distributional parity at a finite set of thresholds $\mathcal Z$ (Figure~\ref{fig:synthetic_data_Zfair}); (iv) \textbf{$\partial{\mathcal{Z}}$-DP} refereed to as {\it Z-fair, range} (Figure~\ref{fig:synthetic_data_Zfair});
(v) \textbf{Strong DP (full distribution matching)}: enforce parity on the whole grid, \emph{e.g.}\ by taking $\mathcal Z=\cY_K$ (Figure~\ref{fig:synthetic_data_Zfair}).
We emphasize that the last baseline represents the ``global'' end of the fairness spectrum, whereas $(\boldsymbol{\ell},\mathcal Z)$-fair and $\mathcal Z$-fair provide \emph{localized} alternatives.

\paragraph{Metrics.}
to evaluate the trade-off between predictive performance and group equity, we report three main metrics. First, we measure the \emph{price of fairness} via the root mean squared error (rmse) between the fair predictor $\hat{f}$ and the unconstrained optimal baseline $\hat{f}^*$:
$\mathrm{rmse} := \sqrt{\frac{1}{n_{\rm test}}\sum_{i=1}^{n_{\rm test}} \big(\hat f^*(X_i,S_i) - \hat f(X_i,S_i)\big)^2}.$
This metric represents the distortion risk $R_D(f)$ minimized in our theoretical results; by construction, the unconstrained model $\hat{f}^*$ yields an rmse of $0.00$. Second, we quantify the \emph{partial demographic parity violation} $\cU_{(\boldsymbol{\ell},\mathcal Z)}(\hat f)$ at the specific thresholds $\mathcal{Z}$ as defined in section~\ref{sec:algorithm}. Finally, we assess the entire outcome range using the kolmogorov-smirnov statistic: $
\mathrm{ks} := \max_{s,s' \in \mathcal{S}} \sup_{t \in [-A, A]} \big| F_{\hat f \mid S=s}(t) - F_{\hat f \mid S=s'}(t) \big|.$

While our optimization targets specific points in $\mathcal{Z}$, the $\mathrm{ks}$ metric allows us to evaluate the impact of these local constraints on the global alignment of the group-conditional predictive distributions.

\paragraph{Implementation details.}
Unless specified otherwise, we set $K=201$, choose $(\boldsymbol{\ell},\mathcal Z)$ based on quartiles ($25^{th}, 50^{th}$, and $75^{th}$ percentiles) of the unconstrained predictor $\hat{f}$ on a calibration set, and use projected subgradient descent to minimize $\widehat H$ in~\eqref{eq:emp_dual}.

\subsection{Focus on $({\boldsymbol{\ell}}, \mathcal{Z})$-fair prediction}
In this first setting, we illustrate the prescriptive capacity of our framework: a practitioner specifies both the thresholds $\mathcal{Z}$ and the target probabilities $\boldsymbol{\ell}$ \emph{a priori}, modeling scenarios where policy dictates acceptance rates or quotas at decision-relevant cutoffs.

Throughout, we take $M=3$ and $\boldsymbol{\ell}=(0.25,0.50,0.75)$ and consider three choices of $\mathcal Z$ (see  Figure~\ref{fig:synthetic_data_lZfair}):
\\
$\bullet$ \textbf{Global}: $\mathcal Z$ is set to the marginal quartiles of the unconstrained scores $\hat f(X,S)$ on the calibration set. This enforces agreement at common thresholds shared across groups.\\
$\bullet$ \textbf{Target-A}: $\mathcal Z$ is set to the quartiles of $\hat f(X,A)$. Since $F_{\hat f\mid S=A}(z_m)=\ell_m$ holds by construction, the constraints effectively force group $B$ to match group $A$ at these thresholds.\\
$\bullet$ \textbf{Target-B}: symmetric choice with $\mathcal Z$ set to the quartiles of $\hat f(X,B)$.

Figure~\ref{fig:synthetic_data_lZfair} shows that these localized constraints can substantially reduce disparities at the prescribed cutoffs while preserving much of the predictive structure away from them. As expected, more prescriptive choices (\emph{e.g.}\ targeting another group at fixed cutoffs) may increase risk when the specified targets are far from the group's natural score distribution.

\begin{figure*}[!htbp]
    \centering
    \includegraphics[width=0.95\textwidth]{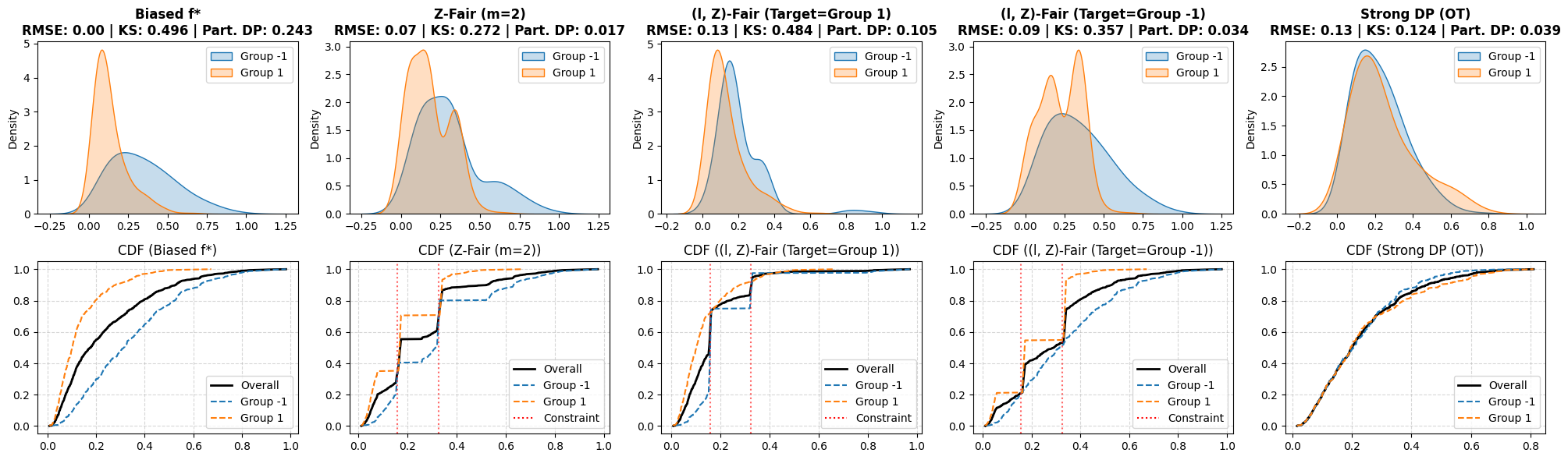} 
    \caption{Comparison of localized constraints and full distribution matching on CRIME data using a LightGBM model with default scikit-learn parameters.}
    \label{fig:qfair_crime}
\end{figure*}

\subsection{Extension to $\mathcal{Z}$-fair and $\partial\mathcal{Z}$-fair prediction}
We now consider $\mathcal Z$-fair constraints, which enforce \emph{partial} distributional parity only at a finite set of thresholds (or over a selected region). Figure~\ref{fig:synthetic_data_Zfair} contrasts enforcing parity at a small number of thresholds (``$Z$-fair, $M=3$'') with a localized range constraint (``$Z$-fair, range'') and with the global strong-DP baseline (full grid matching). Enforcing parity at a few thresholds reduces group differences \emph{where constrained} while allowing more flexibility elsewhere; the range constraint (3rd column) further concentrates the correction within a chosen interval, leaving the tails comparatively less affected. In contrast, full distribution matching (last column) yields near-complete overlap of predictive distributions but can substantially distort predictions.
We refer the reader to Appendix~\ref{app:additionalNum} for additional numerical results relying on the evolution of risk/unfairness \emph{w.r.t.} $M$. 



\subsection{Real-data illustration}
Finally, Figure~\ref{fig:qfair_crime} reproduces the same qualitative behavior on the CRIME dataset (we refer to Appendix~\ref{app:additionalNum} for a description of the dataset) using a LightGBM base regressor (default scikit-learn parameters): localized constraints reduce distributional gaps around selected thresholds while typically incurring a smaller performance penalty than full distribution matching.

\subsection{Overall conclusion.} Our numerical study, both on synthetic and real data highlights that by enforcing constraints only at a finite number of thresholds---or within a selected region of the score distribution---our approach enables \emph{localized} interventions that can be tuned to policy-relevant cutoffs while limiting unnecessary distortion elsewhere. 
These different localized interventions that we considered yield a favorable accuracy--fairness trade-off compared to global matching baselines (OT matching) and confirm our theory.

\begin{figure*}
    \centering
    \includegraphics[width=0.9\textwidth]{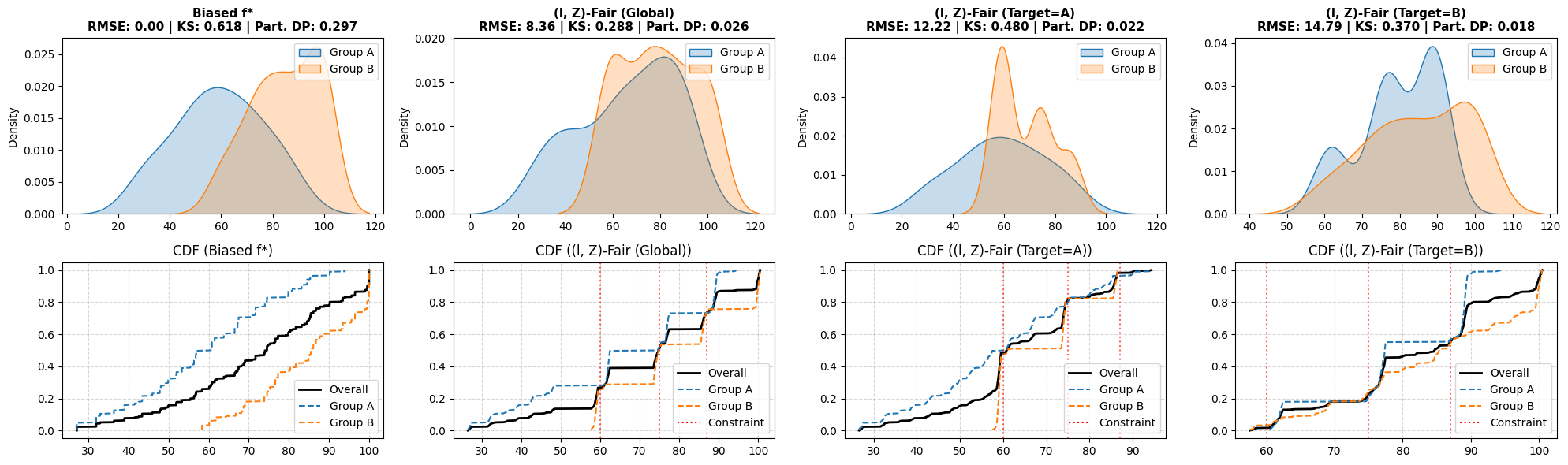}
    \caption{Analysis of $(\ell,\mathcal Z)$-fair methods on synthetic data. We compare three prescriptions for $\mathcal Z$ (Global, Target-A, Target-B) with $M=3$ and $\boldsymbol{\ell}=(0.25,0.50,0.75)$.}
    \label{fig:synthetic_data_lZfair}
\end{figure*}

\begin{figure*}
    \centering
    \includegraphics[width=0.9\textwidth]{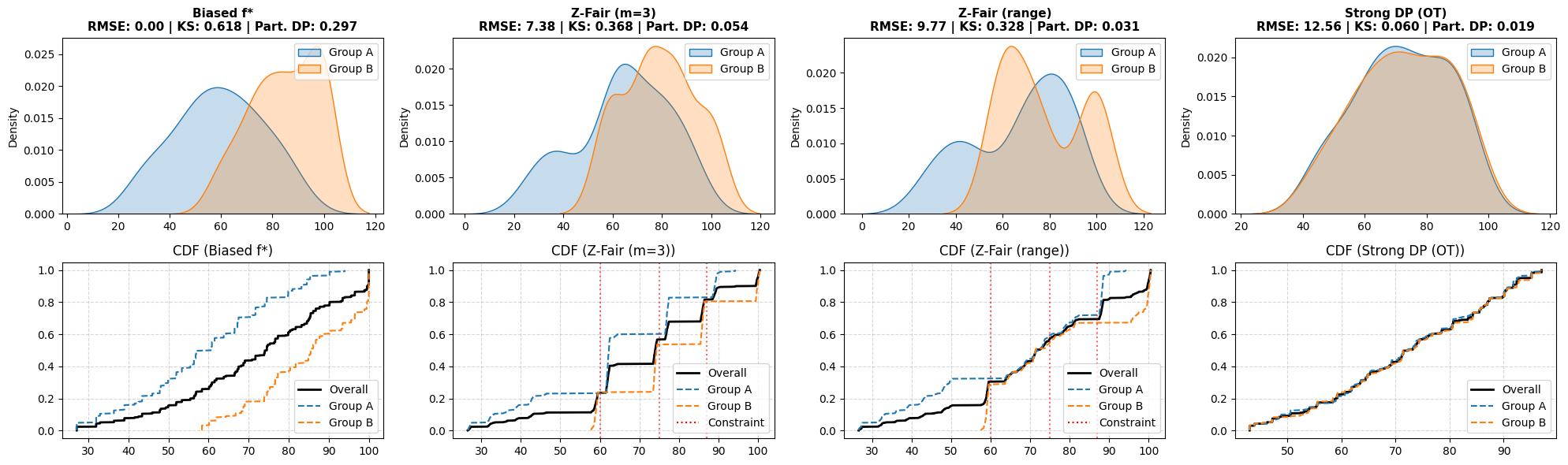} 
    \caption{Analysis of $\mathcal Z$-fair methods on synthetic data. We compare enforcing parity at $M=3$ thresholds, enforcing parity only over a selected range, and full distribution matching (strong DP, full grid).}
    \label{fig:synthetic_data_Zfair}
\end{figure*}


\clearpage

\section*{Impact Statement}

This work contributes to the growing literature on algorithmic fairness in regression by proposing quantile- and threshold-based relaxations of demographic parity. By allowing stakeholders to enforce parity only at selected parts of the predictive distribution (e.g., medians, upper quantiles, or operational cutoffs), the proposed framework can enable more transparent and policy-aligned fairness requirements than full distributional parity, while reducing unnecessary accuracy loss. Potential benefits include improved accountability in high-stakes scoring applications (credit, hiring, risk assessment) and clearer communication of fairness constraints to non-technical decision makers.

At the same time, this approach may have negative societal impacts if misused. First, selecting the levels/thresholds $(\boldsymbol{\ell},\mathcal{Z})$ is a normative choice: poorly chosen targets may hide disparities outside the monitored region of the distribution, or may be used as a superficial ``fairness compliance'' layer without addressing structural harms. Second, the method relies on access to a sensitive attribute $S$ (or reliable proxies) during calibration; collecting, storing, or using such attributes can raise privacy and governance concerns, and may be restricted by regulation or institutional policy. Third, because the procedure is a post-processing step, it can alter score calibration or ranking near cutoffs; if downstream decisions are highly sensitive to small score changes, this may create unexpected incentives or discontinuities.

We emphasize that quantile-based constraints should be deployed only with careful stakeholder consultation and domain expertise. In practice, we recommend: 
(i) reporting the chosen $(\boldsymbol{\ell},\mathcal{Z})$ and conducting sensitivity analyses to alternative choices;
(ii) complementing partial distributional parity with additional diagnostics (e.g., error disparities,
tail-risk metrics, subgroup analyses) to reduce the risk of ``fairness gerrymandering'';
(iii) documenting data collection and privacy safeguards for sensitive attributes; and
(iv) monitoring post-deployment performance to detect distribution shift or new disparities.

Overall, the proposed methodology is intended to provide a tractable and interpretable tool for reducing group-level distributional disparities in regression, but it does not eliminate the need for broader organizational, legal, and societal oversight when automated predictions influence real-world outcomes.

\bibliography{biblio}
\bibliographystyle{icml2026}

\newpage
\appendix
\onecolumn

\begin{center}
\bf{\Large Supplementary Materials}
\end{center}

\vspace*{0.25cm}

\paragraph{Appendix overview.} The first section (Section~\ref{app:additionalNum}) presents additional numerical results complementing those in the main paper. The subsequent sections (Section~\ref{app:proofsofAll}) provide proofs of the theoretical results.

\section{Numerical considerations}
\label{app:additionalNum}

\paragraph{Data description.} The main dataset we consider is  \texttt{CRIME} that contains socio-economic, law enforcement, and crime data about communities in the US with 1994 examples~\cite{redmond2002data}. 
The task is to predict the violent crime rate per population. We consider race-related attributes, in particular the proportion of African-American residents, as sensitive attributes, which obtains 1,032 instances for 
$s=-1$ and 962 instances for 
$s=1$. We split the data into three sets (60\% training, 20\%
hold-out and 20\% unlabeled).

\paragraph{Additional numerical study.}Figure~\ref{fig:z_fair_synthetic_diagram} summarizes the resulting fairness--accuracy trade-off: as constraints become more global (more thresholds and/or full-grid matching), distributional discrepancies decrease (lower $\mathrm{KS}$ and lower constraint violation) at the cost of increased predictive error, whereas localized constraints provide intermediate operating points.

\begin{figure}[!htbp]
    \centering
    \includegraphics[width=0.9\linewidth]{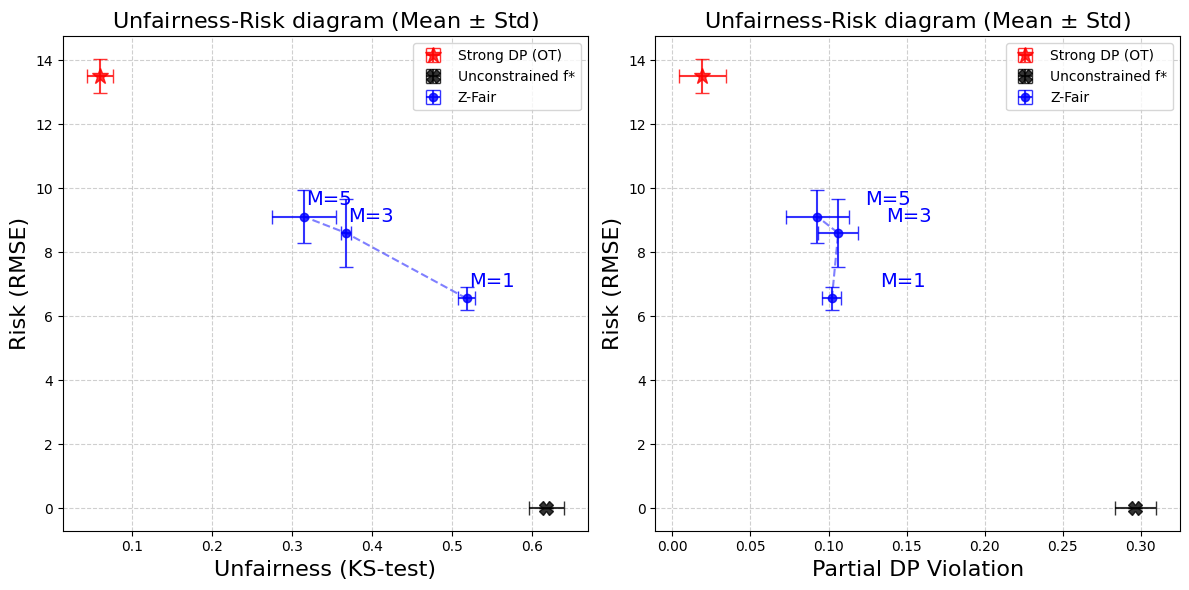} 
    \caption{Fairness--accuracy trade-off on synthetic data for increasingly global constraints.}
    \label{fig:z_fair_synthetic_diagram}
\end{figure}

\section{Proofs of main results}
\label{app:proofsofAll}

This appendix is dedicated to the proofs of the theoretical results. Notice that we omit the proof of Theorems~\ref{thm:OptimFair} and~\ref{thm:proxyBorderPartialDP} since it relies on similar arguments as those in Theorem~\ref{thm:optimLZFair}.

\paragraph{Notation.}
Conditionally on $\cD_n$, define for each $s\in\cS$ the conditional law
$\proba_s(\cdot)=\proba(\cdot\mid S=s)$ and its empirical version based on $\cD_N$,
\[
\hat\proba_s(A):=\frac{1}{N_s}\sum_{i\in I_s}\one\{X'_i\in A\},
\text{ where } I_s:=\{i\in[N]:S'_i=s\}.
\]

\subsection{Proof of Section~\ref{sec:statSetting}}

\begin{proof}[Proof of Theorem~\ref{thm:optimLZFair}]

First, we observe that our minimization problem can reformulated as follows
\begin{equation*}
f^*_{(\boldsymbol{\ell}, \mathcal{Z})-{\rm fair}} \in \argmin_{f \in \mathcal{F}_K} \left\{R(f)-R(f^*) ,   f  \;\;  {\rm is} \;\;    (\boldsymbol{\ell}, \mathcal{Z})-{\rm fair}\right\}.
\end{equation*}

We consider the Lagrangian $\mathcal{L}$ associated to our optimization problem. Let $f \in \mathcal{F}_K$, and $\boldsymbol{\lambda} = (\lambda)_{s \in \mathcal{S},m \in [M]}$, we have that since $R(f)-R(f^*) = \mathbb{E}\left[(f(X,S)-f^*(X,S))^2\right]$
\begin{equation*}
\mathcal{L}\left(f, \boldsymbol{\lambda}\right)
= \mathbb{E}\left[f^*(X,S)-f(X,S))^2\right] + \sum_{s\in \mathcal{S}}\sum_{m \in [M]} 
\lambda_{s,m}\left(\mathbb{P}_{s}\left(f(X,S) \leq z_m\right)-\ell_m\right).
\end{equation*}
We observe that 
\begin{equation*}
\mathcal{L}\left(f, \boldsymbol{\lambda}\right) = 
\sum_{s \in \mathcal{S}}  \mathbb{E}_{s}\left[\pi_s \left(f^*(X,S) - f(X,S)^2\right)
+ \sum_{m \in [M]} \lambda_{s,m} \one_{\{f(X,S) \leq z_m\}}\right] - \sum_{s \in \mathcal{S}}\sum_{m \in [M]} \lambda_{s,m} \ell_m.
\end{equation*}
Now since $f \in \mathcal{F}_K$, we have that
\begin{multline}
\label{eq:eqOptiFair1}
\mathcal{L}\left(f, \boldsymbol{\lambda}\right) = 
\sum_{s \in \mathcal{S}, k \in [K]}  \mathbb{E}_{s}\left[\left(\pi_s \left(f^*(X,S) - y_k\right)^2
+ \sum_{m \in [M]} \lambda_{s,m} \one_{\{y_k \leq z_m\}}\right)\one_{\{f(X,S) = y_k\}}\right] - \sum_{s \in\mathcal{S}, m \in [M]} \lambda_{s,m} \ell_m. 
\end{multline}
Hence, we observe that
\begin{equation*}
f_{\lambda}^* \in \argmin_{f \in \mathcal{F}_K} \mathcal{L}\left(f, \boldsymbol{\lambda}\right),   
\end{equation*}
is characterized pointwise as
\begin{equation*}
f_{\lambd}^*(x,s) = \argmin_{k \in [K]} \pi_s(f^*(x,s)-y_k)^2+ \sum_{m \in [M]} \lambda_{s,m} \one_{\{y_k \leq z_m\}}.
\end{equation*}
Furthermore, we also have
\begin{eqnarray}
\label{eq:eqOptiFair2}
\mathcal{L}\left(f^*_{\lambda}, \lambda\right)
& = & \sum_{s \in \mathcal{S}} \mathbb{E}_{s} \left[\min_{k \in [K]} \left(\pi_s(f^*(x,s)-y_k)^2+ \sum_{m \in [M]} \lambda_{s,m} \one_{\{y_k \leq z_m\}} \right) \right] - \sum_{s \in\mathcal{S}, m \in [M]} \lambda_{s,m} \ell_m  \nonumber \\ 
 & =  & - \left(\sum_{s \in \mathcal{S}} \mathbb{E}_{s} \left[\max_{k \in [K]} \left( -\pi_s(f^*(x,s)-y_k)^2- \sum_{m \in [M]} \lambda_{s,m} \one_{\{y_k \leq z_m\}}\right) \right] + \sum_{s \in\mathcal{S}, m \in [M]} \lambda_{s,m}\ell_m \right).
\end{eqnarray}
Therefore $H : \boldsymbol{\lambda} \mapsto -\mathcal{L}(f^*_{\boldsymbol{\lambda}}, \boldsymbol{\lambda})$ is convex {\it w.r.t.} $\lambda$. Besides, $H$ is coercive. Indeed,
\begin{equation*}
H(\boldsymbol{\lambda}) \geq \max_{k \in [K]} \sum_{s \in \mathcal{S}}\mathbb{E}_{s}\left[-\pi_s(f^*(X,S)- y_k)^2 - \sum_{m \in [M]} \lambda_{s,m} \one_{\{y_k \leq z_m\}}\right]  + \sum_{s \in \mathcal{S}}\sum_{m \in [M]} \lambda_{s,m} \ell_m.
\end{equation*}
Since $|f^*(X,S)| \leq A$ {\it a.s.}, we deduce
\begin{equation*}
H(\boldsymbol{\lambda}) \geq -2 A^2 + \max_{k \in [K]} \sum_{s \in \mathcal{S}}\sum_{m \in [M]} \lambda_{s,m}\left(l_m- \one_{\{y_k \leq z_m\}}\right).
\end{equation*}
From the above inequality, we deduce since for each $m \in [M]$, $y_0 < z_m < y_K$, that 
\begin{equation*} 
H(\boldsymbol{\lambda}) \rightarrow + \infty,    {\rm as},    \left\|\boldsymbol{\lambda}\right\| \rightarrow + \infty. 
\end{equation*}
Therefore, $H$ admits a global minimizer.
Then, we consider the predictor $f^*_{\boldsymbol{\lambda}^*}$ with
\begin{equation*}
\boldsymbol{\lambda}^* \argmin_{\boldsymbol{\lambda}} -\mathcal{L}\left(f^*_{\boldsymbol{\lambda}}, \boldsymbol{\lambda}\right).
\end{equation*}
Under Assumption~\ref{ass:cont}, we have that the function $H$ is differentiable {\it w.r.t.} $\boldsymbol{\lambda}$, and
\begin{equation*}
\partial{H}_{s,m} = -\mathbb{P}_{s}(f^*_{\boldsymbol{\lambda}}(X,S) \leq z_m) + \ell_m.
\end{equation*}
Therefore the first order condition for the minimization over $\boldsymbol{\lambda}$ shows that for $s \in \mathcal{S}$, and $m \in [M]$
\begin{equation*}
\mathbb{P}_{s}(f^*_{\boldsymbol{\lambda}^*}(X,S) \leq z_m) = \ell_m,   
\end{equation*}
which implies that $f^*_{\boldsymbol{\lambda}^*}$ is $(\boldsymbol{\ell}, \mathcal{Z})$-fair.
Finally, we observe that if $f \in \mathcal{F}_K$ is a predictor that is $(\boldsymbol{\ell}, \mathcal{Z}$-fair, we have that 
\begin{equation*}
R(f) -R(f^*) = \mathcal{L}\left(f, \boldsymbol{\lambda}^*\right) \geq \mathcal{L}(f^*_{\boldsymbol{\lambda}^*}, \boldsymbol{\lambda}^*)  = R(f^*_{\boldsymbol{\lambda}^*})-R(f^*).
\end{equation*}
From the above inequality, we deduce that
 $f^*_{\boldsymbol{\lambda}^*} \in \argmin_{f \in \mathcal{F}_K} \left\{R(f)-R(f^*) ,   f    {\rm is}    (\boldsymbol{\ell}, \mathcal{Z})-{\rm fair}\right\}$.
\end{proof}

\begin{proof}[Proof of Proposition~\ref{prop:discrCost}]

First of all, since for each $m \in [M]$, $z_m \in [-A,A]$, we can assume that $\tilde{f}(X,S) \in [-A,A]$.

We define the predictor $T(\tilde{f})$ that is the approximation of $\tilde{f}$ over the set $\mathcal{F}_K$. 
\begin{equation*}
T(\tilde{f})  = y_k ,   {\rm if },   f^*(X,S) \in (y_{k-1}, y_k].     
\end{equation*}
for each $m \in [M]$, since $z_m \in \mathcal{F}_K$,  we have that for each $s \in \mathcal{S}, m \in [M]$
\begin{equation*}
\mathbb{P}_{s}\left(\tilde{f}(X,S) \leq z_m\right) = \mathbb{P}_{s}\left(T(\tilde{f}(X,S) \leq z_m\right) = \ell_m.    
\end{equation*}
Therefore, the predictor $T(\tilde{f})$ is $(\boldsymbol{\ell}, \mathcal{Z})$-fair. Hence 
\begin{equation*}
R(f^*_{(\boldsymbol{\ell}, \mathcal{Z})-{\rm fair}}) \leq R(T(\tilde{f})) =  R(T(\tilde{f})) - R(\tilde{f}) + R(\tilde{f}).   
\end{equation*}
Now, we study the term $R(T(\tilde{f})) - R(\tilde{f})$ in the {\it r.h.s.} of the above inequality.
We have that 
\begin{multline*}
0 \leq  R(T(\tilde{f})) - R(\tilde{f})  = \mathbb{E}\left[2Y\left(\tilde{f}(X,S) - T(\tilde{f}(X,S)\right) + T^2(\tilde{f})(X,S) - \tilde{f}^2(X,S)\right]\\\ = \mathbb{E}\left[(2Y -\tilde{f}(X,S)- T(\tilde{f}(X,S))) \left(\tilde{f}(X,S) - T(\tilde{f}(X,S)\right)\right] 
\end{multline*}
Now since $Y \in[-A,A]$, $T(\tilde{f})(X,S) \in [-A,A]$, and $\tilde{f}(X,S) \in[-A,A]$, we deduce that
\begin{equation*}
 R(T(\tilde{f})) - R(\tilde{f}) \leq 3A \mathbb{E}\left[\left|\tilde{f}(X,S) - T(\tilde{f}(X,S)\right|\right] \leq 3A \max_k \left|y_k-y_{k-1}\right|\leq \dfrac{6A^2}{K-1},
\end{equation*}
which yields the desired result.
\end{proof}

\subsection{Proof of Section~\ref{sec:algorithm}}\label{app:B:section:4}

\subsection*{Proof of Theorem~\ref{thm:unfairnessRate}}
\label{app:proof:unfairnessRate}


For each $k \in [K]$, and $(x,s) \in \mathbb{R}^d\times \mathcal{S}$ we define 
\begin{equation*}
\hat{h}_k(\boldsymbol{\lambda},(x,s)) = \left(-\sum_{m \in[M]} \lambda_{s,m}\left((\one{\left\{y_k \leq z_m \right\}}-\ell_m\right)\right)-\hat{\pi}_s \left(y_k-\bar{f}(x,s)\right)^2.  
\end{equation*}
For each $k \in [K]$, we also define $k(m):= \max \left\{k: y_k \leq z_m \right\}$



Let $s \in \mathcal{S}, m \in [M]$. For $i \in [N]$, we introduce the events 
\begin{equation*}
A_k = \left\{\forall j \neq k,  \hat{h}_j(\lambd,(X,S)) < \hat{h}_k(\lambd,(X,S))\right\},
\end{equation*}
and 
\begin{equation*}
B_k = \left\{\forall j \neq k,   \hat{h}_j(\lambd,(X,S)) \leq  \hat{h}_k(\lambd,(X,S)),    \exists j \neq k,     \hat{h}_j(\lambd,(X,S)) = \hat{h}_k(\lambd,(X,S))  \right\}.
\end{equation*}
We have that 
\begin{equation}
\label{eq:eqUnfair0}
\widehat{\mathbb{P}}_{s}\left(\hat{f}_{(\boldsymbol{\ell}, \mathcal{Z})-{\rm fair}}(X,S) \leq z_m\right)  = \sum_{k \leq k(m)}\widehat{\mathbb{P}}_{s}(A_k) + \widehat{\mathbb{P}}_{s}(B_k).
\end{equation}
Let $g_i \in \partial_{s,m} \max_{k \in [K]} \hat{h}(\lambd, (X,S))$ for $k \in [K]$,  on the event $\left\{\hat{f}_{(\boldsymbol{\ell}, \mathcal{Z})-{\rm fair}}(X,S) = y_k\right\}$, we have
\begin{multline*}
g_i = \partial_{s,m}\hat{h}_k(\lambd, (X,S)) \one_{\left\{A_k\right\}} + \one_{\left\{B_k\right\} } \sum_{j \in [K]} \alpha_{s,m,j}(X,S) \partial_{s,m} \hat{h}_j(\lambda, (X,S)) \one_{\left\{\hat{h}_k(\lambd, (X,S)) = \hat{h}_j(\lambd,(X,S))\right\}},
\end{multline*}
with $(\alpha_{j,s,m}(X,S)_{j \in [K]} \in [0,1]^{K}$ that satisfy
\begin{equation*}
\sum_{j \in [K]} \alpha_{s,m,j}(X,S)  \one_{\left\{\hat{h}_k(\lambd, (X,S)) = \hat{h}_j(\lambd,(X,S))\right\}} = 1 \;\; a.s. 
\end{equation*}
Following similar arguments as in Proof of Theorem~\ref{thm:optimLZFair}, we have that the function $\hat{H}$ is coercive and then admits a minimizer. Since $\hat{\lambd}$ is defined as
\begin{equation*}
\hat{\lambd} \in \argmin_{\lambd} \hat{H}(\lambd) = \sum_{s \in \mathcal{S}} \hat{\mathbb{E}}_{s}\left[\max_{k \in [K]} \hat{h}_k(\lambd,(X,S))\right],
\end{equation*}
Since $\sum_{k \in [K]} \one_{\left\{\hat{f}_{(\boldsymbol{\ell}, \mathcal{Z})-{\rm fair}}(X,S) = y_k\right\}} = 1$,
we deduce from the first order condition that
\begin{equation*}
l_m - \sum_{k \leq k(m)} \widehat{\mathbb{P}}_{s}\left({A_k}\right)  -\sum_{k \leq k(m)}   \widehat{\mathbb{E}}_{s}\left[ \one_{B_k}\sum_{j \leq j(m)} \alpha_{s,m,j}(X,S)\one_{\left\{\hat{h}_k(\lambd, (X_i^s,S_i)) = \hat{h}_j(\lambd,(X,S))\right\}}\right]  = 0
\end{equation*}
Hence, from the above equation, and~\eqref{eq:eqUnfair0} we deduce that
\begin{multline*}
\left|\mathbb{P}_{s}\left(\hat{f}_{(\boldsymbol{\ell}, \mathcal{Z})-{\rm fair}}(X,S) \leq z_m\right) -\ell_m\right| \leq  \left|\left(\mathbb{P}_{s}-\hat{\mathbb{P}}_{s}\right)\left(\hat{f}_{(\boldsymbol{\ell}, \mathcal{Z})-{\rm fair}}(X,S) \leq z_m\right) \right| \\
+ \sum_{k \leq k(m)} \widehat{\mathbb{P}}_{s}\left(\exists j \neq k,     \hat{h}_j(\lambd,(X,S)) = \hat{h}_k(\lambd,(X,S))\right).
\end{multline*}
Therefore, it yields
\begin{equation*}
 \mathcal{U}\left(\hat{f}_{(\boldsymbol{\ell}, \mathcal{Z})-{\rm fair}}\right) \leq  \sum_{s \in \mathcal{S}} \sup_{t \in \mathbb{R}}  \left|\left(\mathbb{P}_{s}-\hat{\mathbb{P}}_{s}\right)\left(\hat{f}_{(\boldsymbol{\ell}, \mathcal{Z})-{\rm fair}}(X,S) \leq t\right) \right| + \sum_{k \in [K]}  \widehat{\mathbb{P}}_{s}\left(\exists j \neq k,     \hat{h}_j(\lambd,(X,S)) = \hat{h}_k(\lambd,(X,S))\right).
\end{equation*}
Now, conditional on $\mathcal{D}_n$, applying using the Dvoretzky--Kiefer--Wolfowitz inequality, \cite{dvoretzky1956dkw} with Massart's sharp constant \cite{massart1990dkw}, and Lemma B.8 in~\citep{chzhen2020fair}, we obtain that ,
\begin{equation*}
\label{eq:eqUnfair2}
\mathbb{E}\left[  \mathcal{U}\left(\hat{f}_{(\boldsymbol{\ell}, \mathcal{Z})-{\rm fair}}\right)\right]  \leq C \sum_{s \in \mathcal{S}}\mathbb{E}\left[\left(\sqrt{\dfrac{1}{N_s}} + \dfrac{K^2}{N_s}\right)\right]. 
\end{equation*}
Finally, using the Lemma~4.1 in~\citep{GyorfyBook2002}, we get the desired result.

\subsection*{Proof of Theorem~\ref{thm:unfairness_hp}}
\label{app:proof:unfairness_hp}

Recall that $\hat f_{(\boldsymbol{\ell},\mathcal Z)\text{-fair}}$ takes values in the finite grid $\cY_K=\{y_1,\dots,y_K\}$.

\paragraph{Step 1: empirical constraints are (essentially) satisfied.}
We claim that, conditionally on $\cD_n$, for all $s\in\cS$ and $m\in[M]$,
\begin{equation}
\label{eq:emp_constraints}
\hat\proba_s\!\Big(\hat f_{(\boldsymbol{\ell},\mathcal Z)\text{-fair}}(X,s)\le z_m\Big)=\ell_m + (\text{tie terms}),
\qquad\text{a.s.}
\end{equation}
with \text{tie terms} of order of $K^2/N$.
Indeed, as in the previous proof, $\hat\lambda$ minimizes the convex objective $\widehat H(\lambda)$ (Eq.~(6) in the main text), hence $0\in\partial \widehat H(\hat\lambda)$. As in the proof of Theorem~\ref{thm:unfairnessRate}, one can compute a subgradient component-wise and obtain
\[
0 \in \partial_{s,m}\widehat H(\hat\lambda)
= \ell_m-\hat\proba_s\!\Big(\hat f_{(\boldsymbol{\ell},\mathcal Z)\text{-fair}}(X,s)\le z_m\Big)
+\text{(tie terms)}.
\]
By the dithering construction, conditionally on $\cD_n$ the random variable $\bar f(X,s)$ has a continuous distribution, which implies that ties in $\displaystyle\argmin_{y\in\cY_K}$ occur with probability zero. Hence the tie terms vanish a.s., yielding \eqref{eq:emp_constraints}.

\paragraph{Step 2: reduce the population violation to a generalization gap.}
Fix $s\in\cS$ and $m\in[M]$. Using \eqref{eq:emp_constraints},
\[
\proba_s\!\Big(\hat f_{(\boldsymbol{\ell},\mathcal Z)\text{-fair}}(X,s)\le z_m\Big)-\ell_m
=
\Big(\proba_s-\hat\proba_s\Big)\!\Big(\hat f_{(\boldsymbol{\ell},\mathcal Z)\text{-fair}}(X,s)\le z_m\Big) + \text{(tie terms)}.
\]


Taking the maximum over $m$ and $s$ yields
\[
\cU_{(\boldsymbol{\ell},\mathcal Z)}\!\Big(\hat f_{(\boldsymbol{\ell},\mathcal Z)\text{-fair}}\Big)
\le \max_{s\in\cS} \sup_{t \in \mathbb{R} } \Big(\proba_s-\hat\proba_s\Big)\!\Big(\hat f_{(\boldsymbol{\ell},\mathcal Z)\text{-fair}}(X,s)\le t\Big) + \text{(tie terms)}.
\]

\paragraph{Step 3: concentration of empirical CDF.}
Conditionally on $\cD_n$, and for a fixed $s$,


we can bound the deviation $\sup_{t\in\R}\big|\hat F_{s}(t)-F_{s}(t)\big|$ using the Dvoretzky--Kiefer--Wolfowitz inequality, \cite{dvoretzky1956dkw}. Using Massart's sharp constant \cite{massart1990dkw}, with probability at least $1-\delta$ we have 
\begin{equation*}
\sup_t |\hat F_s(t)-F_s(t)| \le \sqrt{\frac{\log(2/\delta)}{2N_s}}
\end{equation*}

Using the group-mass assumption $\pi_{\min}>0$ and a standard concentration bound on $N_s$ (e.g., Hoeffding for binomials, \cite{hoeffding1963}), we have on an event of probability at least $1-\delta/2$ that $N_s\ge \frac{1}{2}N\pi_{\min}$ for all $s\in\cS$. Combining and absorbing $\log|\cS|$ and $\log 2$
into constants yields: with probability at least $1-\delta$,
\begin{equation*}
\max_{s \in \mathcal{S}}\left|\sup_{t \in \mathbb{R} } \Big(\proba_s-\hat\proba_s\Big)\!\Big(\hat f_{(\boldsymbol{\ell},\mathcal Z)\text{-fair}}(X,s)\le t\Big)
\right|\leq \mathcal{C}_{\mathcal{S}} \left(\sqrt{\frac{1}{N}}+\sqrt{\frac{\log(1/\delta)}{N}}\right).
\end{equation*}
The above implies the claimed

\[
\cU_{(\boldsymbol{\ell},\mathcal Z)}\!\Big(\hat f_{(\boldsymbol{\ell},\mathcal Z)\text{-fair}}\Big)
\le C_\cS M \,\left(\sqrt{\frac{1}{N}}+\dfrac{K^2}{N}+\sqrt{\frac{\log(1/\delta)}{N}}\right),
\]
which concludes the proof of Theorem~\ref{thm:unfairness_hp}.
\qed


\subsection*{Proof of Theorem~\ref{thm:risk}}
\label{app:proof:risk}

First, for each $\lambda = (\lambda_{s,m})_{s \in \mathcal{S}, m\in [M]} \in \mathbb{R}^{\mathcal{S}M}$, we introduce the predictor $f^*_{\lambda}$ defined as
\begin{equation*}
f^*_{\lambda} \in \arg\min_{f \in \mathcal{F}_K} \mathcal{R}_{\lambda}(f).    
\end{equation*}
It is important to note that the Lagrange multiplier $\lambda^*$  is characterized as 
\begin{equation*}
\lambda^* \in \arg\max_{\lambda \in \mathbb{R}^{\mathcal{S}M}} \mathcal{R}_{\lambda}\left(f_{\lambda}\right).    
\end{equation*}

We start with the following decomposition
\begin{multline}
\label{eq:eqRisk1}
\cR_{\lambda^\star}\left(\hat f_{(\boldsymbol{\ell},\mathcal{Z})\text{-fair}}\right)  - \cR_{\lambda^\star}\left( f^*_{(\boldsymbol{\ell},\mathcal{Z})\text{-fair}}\right)= \\ 
\cR_{\lambda^\star}\left(\hat f_{(\boldsymbol{\ell},\mathcal{Z})\text{-fair}}\right)- \cR_{\hat \lambda}\left(\hat f_{(\boldsymbol{\ell},\mathcal{Z})\text{-fair}}\right) + \cR_{\hat \lambda}\left(\hat f_{(\boldsymbol{\ell},\mathcal{Z})\text{-fair}}\right) -  \cR_{\hat \lambda}\left( f^*_{\hat \lambda}\right) + \cR_{\hat \lambda}\left( f^*_{\hat \lambda}\right) - \cR_{\lambda^*}\left( f^*_{(\boldsymbol{\ell},\mathcal{Z})\text{-fair}}\right)
\end{multline}
By definition of parameter $\lambda^*$, conditional on the data, the last term in the {\it r.h.s.} of the above equation satisfies
\begin{equation*}
\cR_{\hat \lambda}\left( f^*_{\hat \lambda}\right) - \cR_{\lambda^*}\left( f^*_{(\boldsymbol{\ell},\mathcal{Z})\text{-fair}}\right) \leq 0.    
\end{equation*}
Furthermore, since each coordinates of parameters $\hat{\lambda}$, and $\lambda^*$ are bounded by a constant that depends on $A$, we observe that the first term in the {\it r.h.s.} of Equation~\ref{eq:eqRisk1} satisfies
\begin{equation*}
\mathbb{E}\left[\cR_{\lambda^\star}\left(\hat f_{(\boldsymbol{\ell},\mathcal{Z})\text{-fair}}\right)- \cR_{\hat \lambda}\left(\hat f_{(\boldsymbol{\ell},\mathcal{Z})\text{-fair}}\right)\right] \leq CM\mathbb{E}\left[\mathcal{U}\left(\hat f_{(\boldsymbol{\ell},\mathcal{Z})\text{-fair}}\right)\right]. 
\end{equation*}
Therefore, we deduce with Equation~\eqref{eq:eqRisk1} and Theorem~\ref{thm:unfairnessRate} that
\begin{equation}
\label{eq:eqRisk3}
\mathbb{E}\left[\cR_{\lambda^\star}\left(\hat f_{(\boldsymbol{\ell},\mathcal{Z})\text{-fair}}\right)  - \cR_{\lambda^\star}\left( f^*_{(\boldsymbol{\ell},\mathcal{Z})\text{-fair}}\right) \right] \leq C M \left(\sqrt{\dfrac{1}{N}}+\dfrac{K^2}{N}\right) + \mathbb{E}\left[ \cR_{\hat \lambda}\left(\hat f_{(\boldsymbol{\ell},\mathcal{Z})\text{-fair}}\right) -  \cR_{\hat \lambda}\left( f^*_{\hat \lambda}\right)\right].
\end{equation}
Now, we study the second term in the {\it r.h.s.} of the above equation.

From Equation~\eqref{eq:eqOptiFair1}, and~\eqref{eq:eqOptiFair2}, we have that, conditional on the data,
\begin{multline}
\label{eq:eqRisk2}
\cR_{\hat \lambda}\left(\hat f_{(\boldsymbol{\ell},\mathcal{Z})\text{-fair}}\right) -  \cR_{\hat \lambda}\left( f^*_{\hat \lambda}\right) = \\
\sum_{s \in \mathcal{S}}\mathbb{E}_{s}\left[\max_{k \in [K]}\left(-\pi_s\left(f^*(X,S)-y_k\right)^2 - \langle\hat{\lambda}_s, a(y_k)\rangle\right)\right] - \\ \mathbb{E}_{s}\left[\sum_{k \in [K]} \left(-\pi_s\left(f^*(X,S)-y_k\right)^2 - \langle\hat{\lambda}_s, a(y_k)\rangle\right)\one_{\{\hat{f}_{(\boldsymbol{\ell},\mathcal{Z})\text{-fair}} =y_k\}}\right]
\end{multline}
Now, for each $k \in [K]$, and $s \in \mathcal{S}$, we introduce
\begin{equation*}
\tilde{h}_k(X,S) = -\pi_s\left(f^*(X,S)-y_k\right)^2 - \langle\hat{\lambda}_s, a(y_k)\rangle, \;\; {\rm and}, \;\; \hat{h}_k(X,S) = -\hat{\pi}_s\left(\bar{f}(X,S)-y_k\right)^2 - \langle\hat{\lambda}_s, a(y_k)\rangle.
\end{equation*}
Note that we have
\begin{equation*}
f^*_{\hat \lambda}(x,s) \in  \arg\max_{y_k \in \mathcal{F}_K} \tilde{h}_k(x,s) \;\; {\rm and} \;\;  \hat f_{(\boldsymbol{\ell},\mathcal{Z})\text{-fair}}(x,s) \in
\arg\max_{y_k \in \mathcal{F}_K} \hat{h}_k(x,s).
\end{equation*}
Therefore, from Equation~\eqref{eq:eqRisk2}, we deduce that
\begin{equation*}
 \cR_{\hat \lambda}\left(\hat f_{(\boldsymbol{\ell},\mathcal{Z})\text{-fair}}\right) -  \cR_{\hat \lambda}\left( f^*_{\hat \lambda}\right) \leq 2\sum_{s \in \mathcal{S}} \mathbb{E}_{s}\left[\max_{k \in [K]}\left|\tilde{h}_k(X,S) -\hat{h}_k(X,S)\right|\right].  
\end{equation*}
Finally, since each $k \in [K]$, $y_k$, $f^*(X,S)$, and $\bar{f}(X,S)$ are bounded by $A$ we deduce that
\begin{equation*}
\left|\tilde{h}_k(X,S) -\hat{h}_k(X,S)\right|
\leq C_A \left(\pi_s\left|\bar{f}(X,S)-f^*(X,S)\right| + \left|\hat \pi_s - \pi_s\right|\right) \leq C_A\left(\pi_s\left|\hat{f}(X,S)-f^*(X,S)\right| + \left|\hat \pi_s - \pi_s\right| + u \right).
\end{equation*}
Therefore, the last inequality yields
\begin{equation*}
\mathbb{E}\left[ \cR_{\hat \lambda}\left(\hat f_{(\boldsymbol{\ell},\mathcal{Z})\text{-fair}}\right) -  \cR_{\hat \lambda}\left( f^*_{\hat \lambda}\right)\right] \leq   C_A  \left(\mathbb{E}\left[\left|\hat{f}(X,S)-f^*(X,S)\right|\right] + \mathbb{E}\left[\left|\hat \pi_s - \pi_s\right|\right]+u\right).
\end{equation*}
Combining the above equation with Equation~\eqref{eq:eqRisk3} gives the desired result.

\subsection*{Proof of Theorem~\ref{thm:risk_hp}}
\label{app:proof:risk_hp}

We prove a high-probability result, analogue of Theorem~\ref{thm:risk}.

\paragraph{Step 1: a deterministic decomposition.}
Recall
\[
\cR_{\lambda^\star}(f)
= R(f)+\sum_{s\in\cS}\sum_{m\in[M]}\lambda^\star_{s,m}
\Big(\proba_s(f(X,s)\le z_m)-\ell_m\Big).
\]
Since $f^\star_{(\boldsymbol{\ell},\mathcal Z)\text{-fair}}$ satisfies the constraints,
the penalty term vanishes for $f^\star_{(\boldsymbol{\ell},\mathcal Z)\text{-fair}}$, hence
\begin{equation}
\label{eq:penrisk_decomp}
\cR_{\lambda^\star}\!\Big(\hat f_{(\boldsymbol{\ell},\mathcal Z)\text{-fair}}\Big)
-\cR_{\lambda^\star}\!\Big(f^\star_{(\boldsymbol{\ell},\mathcal Z)\text{-fair}}\Big)
=
\underbrace{R(\hat f_{(\boldsymbol{\ell},\mathcal Z)\text{-fair}})
-R(f^\star_{(\boldsymbol{\ell},\mathcal Z)\text{-fair}})}_{(\mathrm{I})}
+\underbrace{\sum_{s,m}\lambda^\star_{s,m}\,\Delta_{s,m}}_{(\mathrm{II})},
\end{equation}
with $\Delta_{s,m}:=\proba_s(\hat f_{(\boldsymbol{\ell},\mathcal Z)\text{-fair}}(X,s)\le z_m)-\ell_m$.

\paragraph{Step 2: control of the penalty term by $\cU_{(\boldsymbol{\ell},\mathcal Z)}$.}
By definition of $\cU_{(\boldsymbol{\ell},\mathcal Z)}$,
$|\Delta_{s,m}|\le \cU_{(\boldsymbol{\ell},\mathcal Z)}(\hat f_{(\boldsymbol{\ell},\mathcal Z)\text{-fair}})$.
Therefore,
\[
|(\mathrm{II})|
\le \Big(\sum_{s\in\cS}\sum_{m\in[M]}|\lambda^\star_{s,m}|\Big)\,
\cU_{(\boldsymbol{\ell},\mathcal Z)}\!\Big(\hat f_{(\boldsymbol{\ell},\mathcal Z)\text{-fair}}\Big).
\]
The quantity $|\lambda^\star_{s,m}|$ depends only on $(\cS,\pi_{\min},A)$ through the dual problem
(Theorem~\ref{thm:optimLZFair}) and is absorbed into the constant $C_\cS$. Therefore, it yields
\begin{equation*}
|(\mathrm{II})| \le \mathcal{C}_{\mathcal{S}} M
\cU_{(\boldsymbol{\ell},\mathcal Z)}\!\Big(\hat f_{(\boldsymbol{\ell},\mathcal Z)\text{-fair}}\Big).    
\end{equation*}

\paragraph{Step 3: control of the risk term by the base regressor error.}
We compare the post-processing based on $f^\star$ and on $\bar f$.
Using boundedness $|Y|\le A$ and $|f|\le A$, the squared loss is $4A$-Lipschitz:
for any two predictors $f,g$,
\begin{equation}
\label{eq:lipschitz_sq}
|R(f)-R(g)|
= \Big|\E\big[(Y-f)^2-(Y-g)^2\big]\Big|
\le 4A\,\E[|f(X,S)-g(X,S)|].
\end{equation}
Since $\bar f=\Pi_{[-A,A]}(\hat f+\xi)$ with $\xi\sim \mathcal{U}nif([0,u])$ independent,
\[
\E\big[|\bar f(X,S)-f^\star(X,S)|\big]
\le \E\big[|\hat f(X,S)-f^\star(X,S)|\big] + \E[|\xi|]
\le \E\big[|\hat f(X,S)-f^\star(X,S)|\big] + u.
\]
The post-processed predictor is obtained by minimizing a pointwise objective of the form
$\hat\pi_s(y-\bar f(x,s))^2+\langle \hat\lambda_s,a(y)\rangle$ over $y\in\cY_K$.
A standard comparison argument (the same as in the proof of Theorem~\ref{thm:risk} in expectation form)
combined with \eqref{eq:lipschitz_sq} yields
\[
(\mathrm{I})
\le C_{\cS,A}\Big(\E\big[|\hat f(X,S)-f^\star(X,S)|\big] + \left|\hat{\pi}_s-\pi_s\right| +  u\Big)
+C_\cS\,\cU_{(\boldsymbol{\ell},\mathcal Z)}\!\Big(\hat f_{(\boldsymbol{\ell},\mathcal Z)\text{-fair}}\Big),
\]
where the additional $\cU$ term accounts for the calibration/generalization gap on the unlabeled sample.

\paragraph{Step 4: plug the high-probability bound on $\cU$.}
Combining Steps 1--3 and applying Theorem~\ref{thm:unfairness_hp} gives that, conditionally on $\cD_n$,
with probability at least $1-\delta$,
\[
\cR_{\lambda^\star}\!\Big(\hat f_{(\boldsymbol{\ell},\mathcal Z)\text{-fair}}\Big)
-\cR_{\lambda^\star}\!\Big(f^\star_{(\boldsymbol{\ell},\mathcal Z)\text{-fair}}\Big)
\le C_\cS \Big(\E[|\hat f-f^\star|] + \left|\hat{\pi}_s-\pi_s\right|+ u\Big)
+ C_\cS\,M\Big(\sqrt{\frac{1}{N}}+\dfrac{K^2}{N} + \sqrt{\frac{\log(1/\delta)}{N}}\Big).
\]
Absorbing constants yields the statement of Theorem~\ref{thm:risk_hp}.
\qed

\end{document}